\documentclass[11pt,letterpaper]{article}
\usepackage[letterpaper,margin=1in]{geometry}
\usepackage{hyperref}
\usepackage{amsmath}
\usepackage{amsfonts}
\usepackage{amssymb}
\usepackage{amsthm}
\usepackage{xspace}
\usepackage[T1]{fontenc}
\IfFileExists{newtxtext.sty}{\usepackage{newtxtext}}{\usepackage{times}}
\usepackage{microtype}
\catcode`\@=11
\makeatletter

\newcommand{\newfontobj}[2]{%
	\def#1{%
		\@ifnextchar[{%
			\newfontobj@opt{#2}%
		}{%
			\newfontobj@noopt{#2}%
		}%
	}%
}

\def\newfontobj@opt#1[#2]#3{%
	\expandafter\DeclareRobustCommand\csname#2\endcsname{%
		\ensuremath{\mathord{\newfontobj@typeset{#1}{#3}}}\xspace
	}%
}

\def\newfontobj@noopt#1#2{%
	\newfontobj@opt{#1}[#2]{#2}%
}

\newcommand{\newfontobj@typeset}[2]{%
	\mathchoice
	{\text{#1{#2}}}%
	{\text{#1{#2}}}%
	{\text{\scriptsize #1{#2}}}%
	{\text{\tiny #1{#2}}}%
}

\makeatother

\newfontobj{\class}{\textsc}
\newfontobj{\advice}{\textit}
\newfontobj{\reg}{\textrm}

\newtheorem{theorem}{Theorem}[section]
\newtheorem{lemma}[theorem]{Lemma}

\newtheorem{corollary}[theorem]{Corollary}

\newtheorem{definition}[theorem]{Definition}

\newtheorem{proposition}[theorem]{Proposition}
\newcommand{\CE}{\mathrm{CE}}
\newcommand{\Pred}{\mathcal{P}}
\newcommand{\enc}[1]{\langle #1\rangle}
\newcommand{\Ex}{\mathbb{E}}

\newcommand{\negl}{\operatorname{negl}}
\newcommand{\K}{\mathrm{K}}
\renewcommand{\H}{\mathrm{H}}
\newcommand{\kpoly}{\K^{\poly}}
\newcommand{\E}[1]{\mathbb{E}_{x\sim #1}}
\class{DTISP}
\class{TC}
\class{NTISP}
\class[kTISP]{$\Sigma_k$TISP}
\class{ID}
\class{NQL}
\class{coNQL}
\class{LIN}
\class{CHECK} \class{TRUE} \class{FALSE}
\class{ALT}
\class{NL} \class{LOGCFL} \class{coNL}
\class{P} \class{NC} \class{FP} \class{UP}
\class{FewP} \class{NP} \class{SAT}
\class{coNP} \class{PSPACE} \class{EXP}
\class{NPSPACE} \class{R} \class{BPP}
\class[sharpp]{\#P} \class[parityp]{$\oplus$P}
\class[stp]{S$_2^p$}
\class[modkp]{Mod$_k$P}
\class[snph]{$\Sigma^p_{s(n)}$}
\class[snpha]{$\Sigma^{p,A}_{s(n)}$}
\class[kph]{$\Sigma^p_{k}$}
\class[modtwop]{Mod$_2$P} \class{SPP}
\class{ACC} \class{AND} \class{MOD} \class{NOT}
\class{OR} \class{GapP} \class{PP}
\class[cep]{C$_=$P} \class{CNF} \class{PH}
\class{IP} \class{NEXP} \class{MAXSNP}
\class{DNF}
\class{DSPACE}
\class{EXPSPACE}
\class{EEXPSPACE}
\class{DTIME}
\class{EEXP}
\class{QBF}
\class{C}
\class{CD}
\class{CND}
\class{I}
\class{AP}
\class{FP}
\class[Fsat]{F$_{sat}$}
\class[SIGMA]{$\Sigma$}
\class{ACCEPT} \class{REJECT}
\class{DTIME}
\class{NTIME}
\class{coNTIME}
\class{NSPACE}
\class{DSPACE}
\class{coNSPACE}
\class{coNTISP}
\class{BPTIME}
\class{BPP}
\class{BQP}
\class{ZPP}
\class{co}
\class{RP}
\class{PEXP}
\advice{poly}
\reg{KL}

\title{How Does Machine Learning Manage Complexity?}

\author{Lance Fortnow\thanks{This research was done while the author was a professor at Illinois Institute of Technology. Much of it was done while a visiting fellow at Magdalen College, University of Oxford.}}
\date{}
\begin{document}

\maketitle
\thispagestyle{empty}

\begin{abstract}
 We provide a computational complexity lens to understand the power of machine learning models, particularly their ability to model complex systems.
 
 Machine learning models are often trained on data drawn from sampleable or more complex distributions, a far wider range of distributions than just computable ones. By focusing on computable distributions, machine learning models can better manage complexity via probability. We abstract away from specific learning mechanisms, modeling machine learning as producing \P/\poly-computable distributions with polynomially-bounded max-entropy.
 
 We illustrate how learning computable distributions models complexity by showing that if a machine learning model produces a distribution $\mu$ that models the distribution generated by a cryptographic pseudorandom generator at least as well as the uniform distribution does, then $\mu$ must be close to uniform.

 We view training as compression that removes the randomness from the data, keeping only its structure, and inference as adding randomness back to generate new outputs. A computationally bounded model cannot distinguish pseudorandomness from true randomness, so training removes both, treating anything it cannot efficiently predict as noise. The same holds at inference, where the randomness added back can be, and in practice is, pseudorandom.

 In the appendix, we make the connection to compression precise and, using a characterization of pseudoentropy by Vadhan and Zheng, show that the best loss achievable by an efficient model is the next-bit pseudoentropy of the data.
 \end{abstract}

\newpage
\setcounter{page}{1}

\section{Introduction}

\begin{quote}
Ignorance of the different causes involved in the production of events, as well as their complexity, taken together with the imperfection of analysis, prevents our reaching certainty about the vast majority of phenomena. Thus there are things that are uncertain for us, things more or less probable, and we seek to compensate for the impossibility of knowing them by determining their different degrees of likelihood. So it was that we owe to the weakness of the human mind one of the most delicate and ingenious of mathematical theories, the science of chance or probability.

\hfill -- Pierre-Simon Laplace~\cite{laplace}
\end{quote}

Modern machine learning has modeled complex behavior in ways that go far beyond our traditional logic approaches, making tremendous advances in language translation, computer vision, Chess and Go playing, protein folding, automated programming and mathematical proofs to name just a few.

In this paper we give a complexity lens to machine learning models, abstracting the model away from the specific technologies. We then argue that paradoxically these models derive their strength from the weakness of their distributions, modeling complex behavior by probabilistic outcomes, much the way Laplace has argued that humans do the same.

One way to see this is through compression. Training a model is like making frozen concentrate~\cite{frozen}. It squeezes the randomness out of the data, leaving the structure, and inference adds randomness, or in practice pseudorandomness, back to produce new outputs. But a computationally bounded model cannot tell true randomness from pseudorandomness, so it squeezes out both. Anything that looks random to an efficient observer, whether it came from coin flips or from a complex process that merely appears random, like a cryptographic generator, gets treated as noise. The model doesn't need to understand every complex process behind its data. It only needs to capture what it can efficiently predict and leave the rest to chance.

In Section~\ref{klsec}, we review the argument that describing a string with a simple distribution leads to minimizing the Kullback-Leibler divergence of the original distribution $D$ from a learned computable distribution $\mu$, and view training as compressing out randomness and inference as adding it back.

In Section~\ref{distsec}, we abstract the mechanisms of machine learning to model the output of a machine learning model as a \P/\poly-computable distribution with polynomially-bounded max-entropy. We make separate arguments for the various aspects. 
\begin{itemize}
	\item Computable distributions come from viewing next-token prediction as \P-computable conditional probability functions which are equivalent to \P-computable distributions.
	\item We can achieve polynomial time even though neural nets have low depth, through either reasoning or coding. 
\item 
The heavy use of nonuniformity, expressed through the weights of the model, captures data, time and capabilities. 
\item Finally, we argue that the softmax function guarantees polynomially-bounded max-entropy, i.e., that every outcome has at least exponentially small probability of occurring. 
\end{itemize} 

In Section~\ref{randomsec}, we present our main theorem. Let $D$ be a distribution generated by a cryptographic pseudorandom generator, $U$ the uniform distribution, and $\mu$ a $\P/\poly$-computable distribution with polynomially-bounded max-entropy that models $D$ at least as well as $U$ does, i.e., $\KL(D\|\mu)\leq\KL(D\|U)$. We show this yields $\KL(U\|\mu)<1/n^c$ for all $c>0$ and all sufficiently large $n$. In particular this holds for the best model of any fixed polynomial size. While $D$ and $U$ are computationally indistinguishable, when you learn a computable distribution $\mu$ from $D$ then $\mu$ and $U$ are information-theoretically indistinguishable. 

If we remove the max-entropy condition, we can still show that $U$ and $\mu$ are statistically close.

We put it all together in Section~\ref{togethersec}, arguing that the power of modern machine learning models comes both from their ability to generate random guesses instead of specific answers and from limiting the hypotheses to computable distributions, which forces the models to represent complex behavior by probabilities. 

In the appendix we connect the model to two other areas. Appendix~\ref{codingapp} makes Section~\ref{klsec} precise: a model compresses strings to a length given by its log loss, so training a neural net minimizes compressed length, though such model-based compression can be far from the shortest efficient descriptions. Appendix~\ref{pesec} uses a characterization of pseudoentropy due to Vadhan and Zheng~\cite{VZ12} to show that the best loss any efficient model can achieve is the next-bit pseudoentropy of the data, a notion introduced by Haitner, Reingold and Vadhan~\cite{HRV13}. Theorem~\ref{prgthm} is the special case of a pseudorandom generator.

\section{Previous Work}

Theoretical computer science has studied learning since the early days of the field. In 1967, Mark Gold~\cite{Gold1967} developed inductive inference and in the 1980s Leslie Valiant~\cite{Valiant-PAC} described PAC (Probably Approximately Correct) learning. Both these paradigms focused on learning deterministic algorithms and not probabilities. Kearns and Valiant~\cite{Kearns-Valiant} showed that learning circuits in general in the PAC model would break cryptography. We compare our model with PAC learning and with learning distributions below.

\subsection{How Our Model Differs from PAC Learning}
\label{pacsec}

Leslie Valiant's PAC model~\cite{Valiant-PAC} has been the standard theoretical model of learning for four decades. A learner sees examples labeled by an unknown Boolean function from some known class and must efficiently find a hypothesis that agrees with that function on most inputs. The PAC model has led to a rich theory, but it doesn't capture how modern machine learning actually works.

Large language models don't learn a function from labeled examples. They learn a distribution from unlabeled text, closer to the model of learning distributions developed by Michael Kearns, Yishay Mansour, Dana Ron, Ronitt Rubinfeld, Robert Schapire and Linda Sellie~\cite{KMRRSS} that we discuss in Section~\ref{kmrsec}. One could view predicting the next token as a labeling problem, but the label is itself random given what came before, more like the probabilistic concepts of Kearns and Schapire~\cite{KearnsSchapire} than a Boolean function.

The measure of success differs as well. PAC learning counts mistakes. We measure KL divergence, equivalent to the cross-entropy loss used to train neural nets. A PAC hypothesis must commit to an answer, while our models can hedge their bets, rewarded for honest uncertainty and heavily penalized for being confidently wrong. This hedging is what lets the models treat what they can't predict as random.

PAC learning assumes the target comes from a known class, or in the agnostic version compares the learner to the best hypothesis in a class. We make no assumptions on the data, which in the real world may come from a sampleable or even more complex distribution, and instead ask what the best efficient model looks like. PAC learnability also ties together whether a good hypothesis exists and whether we can efficiently find it. We focus on the first question, viewing the weights as nonuniform advice (Section~\ref{distsec}), and leave the second for Section~\ref{learningsec}.

Finally, cryptography plays a very different role. Kearns and Valiant~\cite{Kearns-Valiant} showed that if cryptography is secure, we cannot PAC learn polynomial-size circuits, a failure of learning. In our model, a model trained on pseudorandom data simply treats it as random, which is the best any efficient model can do (Theorem~\ref{prgthm} and Appendix~\ref{pesec}). Cryptography doesn't break the learning. It shows the learning working as intended, managing complexity through probability.

\subsection{How Our Model Differs from Learning Distributions}
\label{kmrsec}

The closest model to ours is the distribution learning model of Kearns, Mansour, Ron, Rubinfeld, Schapire and Sellie~\cite{KMRRSS}. Like us, they measure a hypothesis by its KL divergence from the target distribution, and they use the uniform distribution, with divergence at most $n$, as the baseline for weak learning. They distinguish between hypotheses that are \emph{generators}, which sample from a distribution, and \emph{evaluators}, which compute the probability of a string.

Their model, though, follows the PAC tradition. The target distribution comes from a known restricted class, such as distributions generated by small circuits of OR or parity gates, mixtures of Hamming balls or probabilistic finite automata, and the goal is an efficient algorithm that finds a hypothesis close to the target from samples. We make no assumption on where the data comes from, and instead ask what the best efficient model looks like, whether or not the target itself has an efficient description.

Our hypotheses are also more restrictive than either of theirs. A next-token predictor gives the probability of each bit conditioned on the bits before it, equivalent to a computable distribution (Section~\ref{distsec}). An evaluator need only give the probability of a full string, which doesn't by itself let you complete a prompt, and a generator can sample but not evaluate probabilities. Restricting to computable distributions is exactly what forces our models to treat what they can't predict as randomness.

The role of cryptography differs as well. Kearns, Mansour, Ron, Rubinfeld, Schapire and Sellie show that if one-way functions exist, distributions generated by polynomial-size circuits are not efficiently learnable with either an evaluator or a generator. Their hard distributions pair a random $x$ with $f_k(x)$ for a pseudorandom function $f_k$, essentially the example in Section~\ref{learningsec} of a distribution whose best computable model can't be found efficiently. Each of these distributions has a small evaluator; the hardness lies in finding it. Our results say something different. For the output of a pseudorandom generator, any efficient model that does as well as uniform must be close to uniform (Theorem~\ref{prgthm}), and no efficient model does noticeably better (Appendix~\ref{pesec}). There is nothing better for the learner to find.

\subsection{Other Related Work}

Shuichi Hirahara and Mikito Nanashima~\cite{HN-Pessiland} show that in a world without cryptography one can learn certain distributions well. However, their work only learns distributions with relatively small initial state, whereas we want a theory that handles complex distributions.

Bruno Cavalar, Eli Goldin, Matthew Gray, Taiga Hiroka, Min-Hsiu Hsieh and Tomoyuki Morimae~\cite{CGGHHM25} show that every efficiently sampleable distribution can be efficiently verified from samples if and only if one-way functions do not exist. Their verifier accepts when the samples have time-bounded Kolmogorov complexity close to $\log 1/D(x)$, the condition in Section~\ref{klsec} for $x$ to be a random example of a distribution. For hardness, they replace a sampler's randomness with the output of a pseudorandom generator, much as we do in Section~\ref{randomsec}. A verifier can't tell the difference, just as a model can't.

Marc Finzi, Shikai Qiu, Yiding Jiang, Pavel Izmailov, Zico Kolter and Andrew  Wilson~\cite{epiplexity} also explore time-bounded distributions in the context of machine learning. They focus on distributions $\mu$ where $\mu$ is both sampleable and $\mu(x)$ is computable, a concept that is broader than computable distributions but doesn't necessarily lead to the conditional sampling that neural nets provide. Their results are similar in spirit to our Section~\ref{randomsec}. Yilun Xu, Shengjia Zhao, Jiaming Song, Russell Stewart and Stefano Ermon~\cite{XZSSE20} earlier defined usable information relative to a class of efficient predictors, a precursor of epiplexity.

Our results connect to computational notions of entropy~\cite{HILL,BSW,HRV13,VZ12} and to prior and concurrent work relating next-token prediction, compression and pseudoentropy~\cite{HMS23,HLN24,MMP26,ABLMRT26,DRD+24} (see Appendix~\ref{pesec}). They also connect to outcome indistinguishability~\cite{OI,GHKRW23}, which relates loss minimization to computational indistinguishability in the spirit of Theorem~\ref{decompthm}.

This paper tries to distill ideas from a vast literature on machine learning to abstract to a relatively simple model of machine learning that we can directly study. In each of the sections that follow, we will cite some of the seminal work that shaped the author's thinking for this paper. 

\section{Preliminaries}

This paper uses ideas, concepts and tools from information theory, Kolmogorov complexity and cryptography. For details on the results mentioned in this section and general background, the author recommends the textbooks of Cover and Thomas~\cite{CoverThomas}, Li and Vit{\'a}nyi~\cite{LiVi} and Goldreich~\cite{Goldreich1} respectively. 

A function $f(n)$ is negligible if $f(n)=o(1/n^k)$ for all $k$. We abuse notation to say $f(n)\leq \negl(n)$ if $f(n)$ is negligible. All logarithms are taken base 2.

\subsection{Nonuniform Circuits}
Every polynomial-time computable language $L$ can be computed by a family $C_0, C_1, \dots$ of polynomial-size circuits. In this framework, $C_n$ has $n$ inputs, and the number of gates of $C_n$ is bounded by $n^c$ for some constant $c$. A string $x = x_1 \dots x_n$ is in $L$ if and only if $C_n(x_1, \dots, x_n)$ outputs 1.

The converse is not true; there are languages $L$ that have polynomial-size circuits but cannot be computed in polynomial time. This is because each $C_n$ might not have any relation to $C_m$ for $n \neq m$. We can achieve equivalence by adding a uniformity condition, requiring an easily computable function $f$ such that $f(1^n) = C_n$. Alternatively, we can keep the nonuniformity to define a larger class called $\P/\poly$.

\subsection{Information Theory}
In this paper all distributions will be over $\{0,1\}^n$. We state the definitions and known relationships for polynomial-time but they all hold if we allow nonuniformity as well.

A distribution $\mu$ is \emph{polynomial-time sampleable} if there is a probabilistic polynomial-time algorithm $A$ such that $\mu(x)$ is the probability that $A(1^n)=x$.

A distribution $\mu$ is \emph{polynomial-time computable} if the CDF function $f(x)=\sum_{y\leq x}\mu(y)$ is computable in polynomial time. 

If $\mu$ is polynomial-time computable then it is sampleable and $\mu(x)=f(x)-f(x-1)$ is also computable in polynomial time.

We also consider $\P/\poly$-computable and $\P/\poly$-sampleable distributions where we replace polynomial time by nonuniform polynomial time. The same relationships hold. 

If there are cryptographic one-way functions then there are \P/\poly-sampleable distributions that are not \P/\poly-computable, for example consider the distribution $(f(x),x)$ where $f$ is a one-way function and $x$ is chosen at random. This distribution is sampleable and if it were computable we could find a preimage of $f(x)$. 

For a distribution $\mu$, the entropy of $\mu$, denoted $\H(\mu)$, is $\sum_{x\in\{0,1\}^n}\mu(x)\log 1/\mu(x)$. The max-entropy of $\mu$ is $\max_x \log 1/\mu(x)$. Polynomially-bounded max-entropy\footnote{Some of the cryptography literature uses max-entropy for the logarithm of the support size. Here it means the largest value of $\log 1/\mu(x)$.} means the max-entropy is bounded by $n^k$ for some $k$, equivalent to saying $\mu(x)\geq 2^{-n^k}$. 

For two distributions $\mu$ and $\tau$ the KL-divergence of $\tau$ from $\mu$ is
\[\KL(\tau\|\mu)=\E{\tau}[\log(\tau(x)/\mu(x))]\]
The KL-divergence is not symmetric but it is always nonnegative and equal to zero only when $\tau=\mu$. It measures how well the distribution $\mu$ models $\tau$ as seen by the cross-entropy decomposition
\[\KL(\tau\|\mu)=\E{\tau}[\log(1/\mu(x))]-H(\tau).\]

We can also measure the statistical difference between $\tau$ and $\mu$ as 
\[\Delta(\tau,\mu)=\frac12\sum_x|\tau(x)-\mu(x)|\]
which is symmetric, nonnegative and zero only if $\tau=\mu$.

Pinsker's inequality allows us to bound the statistical difference from the \KL-divergence. \[
\Delta(\tau,\mu) \le \sqrt{\frac{\ln 2}{2}\,\KL(\tau\|\mu)}
\]
The other direction does not hold as the KL-divergence can get very large if $\mu(x)\ll\tau(x)$ for some $x$, even infinite if for some $x$, $\tau(x)>\mu(x)=0$.

\subsection{Cryptography}
A \emph{cryptographic pseudorandom generator} is a function $G$ whose output on uniform inputs is not efficiently distinguishable from a uniform distribution~\cite{Y}. Formally, $G$ is computable in polynomial time and takes seeds of length $\ell$ to outputs of length $n>\ell$ such that for all probabilistic \P/\poly algorithms $A$,
\[|\Pr_{s\in\{0,1\}^\ell}(A(G(s))=1)-\Pr_{r\in\{0,1\}^n}(A(r)=1)|\leq\negl(n).\]
Håstad, Impagliazzo, Levin and Luby~\cite{HILL} show that one can create a cryptographic pseudorandom generator from any one-way function.

\subsection{Kolmogorov Complexity}
For a string $x$, the Kolmogorov complexity of $x$, $\K(x)$, is the length of the shortest program that outputs $x$. We also consider $\kpoly(x)$ where the program runs in time a fixed polynomial in $n=|x|$. We will be using Kolmogorov complexity informally in this paper so we will ignore quantifiers on the polynomials and small additive factors in the inequalities.

\section{Minimizing Kullback-Leibler Divergence}
\label{klsec}

In this section we show that finding a $\mu$ that minimizes the Kullback-Leibler divergence of a given distribution $D$ from $\mu$ arises naturally from the goal of finding the simplest explanation for $D$. This section is based on ideas from Jorma Rissanen~\cite{Rissanen}, Ray Solomonoff~\cite{Solomonoff}, and Ming Li and Paul Vitányi~\cite{LV-MDL}.

\subsection{Structure and Randomness}

For any finite set $S$ containing $x$, $K(x)\leq K(S)+\log|S|$, i.e., you can describe $x$ by the description of a set and the index of $x$ in that set.

If we have equality then $x$ is a random element of $S$. We can achieve equality with $S=\{x\}$, but consider the simplest set $S$ (minimizing $K(S)$ which maximizes $|S|$) where we have equality. That breaks $x$ into a structured part, the description of $S$, and a random part, the index of $x$ in $S$. By an Occam's razor argument, the description of $S$ is the best explanation for $x$.

There is no efficient or even computable procedure for finding the best $S$. We'll now shift our focus to polynomial time and distributions.

Let $\mu$ be a P-computable distribution with CDF function $f(x)$. Then for all $x$,
\[\kpoly(x)\leq \log(\frac{1}{\mu(x)})+\kpoly(\mu)\] up to small additive factors.
There is some integer $w$ such that $f(x-1)<w/2^k\leq f(x)$ where $k=1+\lceil \log(1/\mu(x))\rceil$. Given a description of $\mu$, $k$ and $w$ we can find $x$ using binary search.

If we have equality, then $\mu$ is a structured part of which $x$ is a random example, i.e., random in $\mu$. We can get equality by a $\mu$ that puts all its weight on $x$. We want to find the simplest $\mu$, i.e., that minimizes $\kpoly(\mu)$ while achieving near equality.

Suppose $x$ comes from a distribution $D$ and we want to create the distribution $\mu$ before $x$ is drawn. Taking expected values of both sides
\[\E{D}(\kpoly(x))\leq\E{D}(\log \frac{1}{\mu(x)})+\kpoly(\mu)=\H(D)+\KL(D\|\mu)+\kpoly(\mu)
\]
by the cross-entropy decomposition.

Since $\H(D)$ does not depend on $\mu$, by the Occam's razor principle we can get the best model for $D$ by finding the $\P$-computable $\mu$ that minimizes $\KL(D\|\mu)+\kpoly(\mu)$.

Let's connect this with neural nets. The value $\kpoly(\mu)$ is dominated by the number of weights in the model, a hyperparameter to be optimized. Fixing that size, the model will try to find the weights that minimize $\KL(D\|\mu)$, equivalent to minimizing the cross-entropy loss, the standard training objective of neural nets.

\subsection{Frozen Concentrate}
\label{frozensec}

This two-part view gives a picture of what training and inference actually do. Ted Chiang~\cite{Chiang23} described ChatGPT as a blurry JPEG of the web, a lossy compression of its training data. The author has argued~\cite{frozen} that a better analogy is frozen concentrate. The text we train on is just one random sample of the many texts that could have been written. Training removes that randomness, the way water is removed to make concentrate, leaving the structured part $\mu$. Inference adds the randomness back, producing new text that isn't the original but is as good a draw from $\mu$ as the original was.

The compression above makes this precise. Decompression takes the random part $w$ and, using $\mu$, recovers $x$ by binary search, the same arithmetic coding that makes language models effective lossless compressors~\cite{DRD+24}. If instead we start the binary search from a uniformly random number $v\in[0,1)$, we find the $x$ with $f(x-1)<v\leq f(x)$, which happens with probability exactly $f(x)-f(x-1)=\mu(x)$. So sampling from a model is decompression applied to fresh randomness: we keep the concentrate and add new water. Next-token sampling does exactly this, one bit at a time, using the conditional probabilities of Section~\ref{distsec}.

Pseudorandomness goes both ways. Just as training can't distinguish pseudorandomness from randomness, neither can anyone looking at the output of inference. If we replace the random $v$ with the output of a cryptographic pseudorandom generator, the resulting samples are computationally indistinguishable from true samples of $\mu$, since any efficient test that told them apart would break the generator. Real-world neural nets do exactly this. Inference draws its randomness from a pseudorandom number generator, often with a seed the user can set to reproduce an output, though typically not one designed to be cryptographically secure. The water we add back is itself pseudorandom.

The analogy also tells us what the model can't keep. Anything the model can't distinguish from randomness gets removed along with the water. In Section~\ref{randomsec} we show that if part of the data is generated pseudorandomly, a model can't tell the difference from truly random data, so the output of a pseudorandom generator is all water. Appendix~\ref{codingapp} makes the compression argument precise (Theorem~\ref{codingthm}) and shows the cost of this: for pseudorandom distributions, the shortest efficient description only needs the seed, but compression based on any efficient model is nearly as long as the data itself (Theorem~\ref{septhm}), so the Occam's razor intuition above has its limits. Appendix~\ref{pesec} measures exactly how much water a model must add back: the best compressed length any efficient model can achieve is the next-bit pseudoentropy of the data.

\section{Neural Nets and Computable Distributions}
\label{distsec}

In this section, we argue that neural nets are roughly equivalent to $\P/\poly$-computable distributions with polynomially-bounded max-entropy. We examine each aspect, computable distribution, polynomial-time, nonuniformity and max-entropy, below.

The input length $n$ corresponds to the size of the context window, typically about one million tokens for frontier models at the time of this writing~\cite{anthropic2026opus46systemcard}. Each token averages about 3-4 bytes.

\subsection{Computable Distributions}

Neural nets predict the next token based on the tokens seen so far (the so-called ``stochastic parrots''~\cite{StochasticParrots}). Learning the next token function makes the learning process tractable. If we tried to train on entire strings, they would occur with such low probability that we wouldn't have enough samples to make reasonable hypotheses. In this section, we show that learning next-token predictors means we limit our hypotheses to computable distributions. 

The transformer model~\cite{Vaswani} gives neural nets the ability to see their input in order. Mathematically then we can describe these next token predictors as conditional probability functions. If we stick to the binary regime, a \emph{conditional probability function} for a distribution $\mu$ on $n$ bits is a function $f$ from $\{0,1\}^{<n}$ to a value $p\in[0,1]$ such that for any $y$ in $\{0,1\}^{<n}$, $f(y)$ is the conditional probability of the next bit being 1 given the prefix $y$. Formally,
\[
f(y) = \Pr_{x \sim \mu}\big[x_{|y|+1}=1 \mid x_{1\ldots |y|}=y\big].
\]
In terms of the probabilities assigned by $\mu$,
\[
f(y) = \frac{\mu(y1)}{\mu(y0)+\mu(y1)},
\]
whenever $\mu(y0)+\mu(y1) > 0$. Here $\mu(z)$ denotes the probability that the string drawn from $\mu$ begins with the prefix $z$.

Neural nets output tokens instead of bits but we can think of them just outputting one bit at a time. 

If we don't restrict the complexity of $f$, then conditional probability functions can represent any distribution. In this section we explore what happens when we do restrict the complexity.

A \P/\poly-computable conditional probability function is a conditional probability function where the function $f$ is computable in \P/\poly. 

A distribution $\mu$ is polynomial-time computable exactly when $\mu$ can be described by a polynomial-time computable conditional probability function. We give a formal proof of this connection in Appendix~\ref{martsec} (Theorem~\ref{martthm}). The same holds for $\P/\poly$-computable distributions and $\P/\poly$-computable conditional probability functions with a similar proof.

In a large language model, we need to output a sequence based on a prompt, so we need a notion of conditionally sampleable. 

A distribution $\mu$ is \textit{$\P/\poly$-conditionally sampleable} if there is a $\P/\poly$-computable probabilistic algorithm $A$ such that $A(1^n,y)$ outputs $x$ with probability $\Pr_\mu(x\ |\ y\ \sqsubseteq x)$ if $y\ \sqsubseteq x$ ($y$ is a prefix of $x$) and zero otherwise.

You can computably sample from a conditional probability function by starting at $y$ and choosing the next bits according to the prescribed probabilities. So every $\P/\poly$-computable distribution is $\P/\poly$-conditionally sampleable. 

The converse isn't likely true exactly, but we can approximate the values of the conditional probability function from a conditionally-sampleable distribution, so they are roughly equivalent.

In this paper, we assume that machine learning produces $\P/\poly$-computable distributions, both because next token predictors act this way and because we need such distributions if we want to give answers to inputs.

\subsection{Polynomial-Time}

We need to argue that we can consider \P-computable distributions when the frontier machine learning models have low depth, typically about a hundred layers. The neural nets by themselves do not have large depth, but the frontier models allow reasoning and executing code, both of which allow us to achieve polynomial depth. This section builds on the chain-of-thought reasoning described by Jason Wei, Xuezhi Wang, Dale Schuurmans, Maarten Bosma, Brian Ichter, Fei Xia, Ed Chi, Quoc V. Le and  Denny Zhou~\cite{ChainOfThought}.

\subsubsection*{Reasoning}

Reasoning models allow neural nets to write down information and then take next steps based on what they wrote down. Alan Turing~\cite{Turing} defined his eponymous machine after considering how a then-human computer processed information. Thus a reasoning model acts like a Turing machine. 

More formally, the next state of a Turing machine can be computed from the current state via a small-depth circuit, even in $\NC^0$. So a neural net could via reasoning in polynomial-time compute any polynomial-time function. Keeping all these configurations will take up a large part of the context window but we only need to remember the last one, similar to the way some models compress the discussion so far so they can continue their work without going past the context window.

This keeps the distribution computable as long as the reasoning steps are part of the output, or the steps that get discarded are deterministic. If a model makes hidden random choices while reasoning and then throws them away, the distribution of its visible output need only be sampleable. For example, it could privately pick a random seed $s$ and output $G(s)$. Our model applies to the visible sequence, including any visible reasoning, and does not capture hidden randomized reasoning.

\subsubsection*{Coding}

More directly, neural nets can write and execute their own code. For example, neural nets are notoriously bad at arithmetic operations like multiplying large numbers, but they know how to write code to multiply, not unlike how humans would tackle the same task.

A neural net could simply write code that simulates a polynomial-size circuit based on its own weights. 

\subsection{Nonuniformity}

In computational complexity, nonuniformity allows having different programs for inputs of different lengths. The concept of nonuniformity comes up often in complexity, most notably for circuit complexity and pseudorandomness. Nonuniformity makes it more difficult to prove lower bounds as it can get in the way of straightforward diagonalization. For example, we know $\P$, polynomial-time computable languages, are strictly contained in $\EXP$, languages computable in exponential time, but for all we can prove, nonuniform polynomial-time could contain non-deterministic exponential time.

Nonuniformity was historically seen more as a technicality than as something that gave circuits significantly more power. Nearly all programs used in practice are highly uniform with a single algorithm that works on all input lengths. For example, standard algorithms for sorting, shortest path and matching work on every input length without modification. Rarely did we use nonuniformity in our programs, other than for random bits for randomized algorithms.

Machine learning puts a new face on uniformity. Machine learning typically works in two phases: the \emph{training phase} where a model's weights are set based on data, and the \emph{inference phase} where we use the model with its weights to make predictions based on new data. These weights are the nonuniform advice seemingly needed for problem solving. The latest frontier models are at their core neural nets with about a trillion parameters. Let's consider the input length to be the size of the context window, about a million tokens. The weights are somewhere between a square and a cube of the input length, like polynomial advice.

We have to be careful talking about nonuniformity and running time when working at a fixed length. Nevertheless there are good reasons why we should consider the weights as nonuniform advice.

\begin{enumerate}
	\item \textbf{Running Time}
	
	The inference phase typically runs quickly, a few seconds, or minutes in a reasoning model. A fully trained model is a straightforward circuit and computing its outcome can be done quickly, especially in parallel in large data centers. On the other hand, the training phase, where the model learns the weights via gradient descent and back propagation, often takes months training on hundreds of trillions of tokens. Parallelizing the training phase is possible but tricky and limited. It would be highly infeasible to run a training phase every time we want to do inference. It would be a stretch to say the training phase takes time exponential in the input length, but at best it is a much larger polynomial than the running time of the inference phase.
	
	\item \textbf{Data}
	
	As mentioned above, the neural net is trained on hundreds of trillions of tokens. The trained weights of the neural net directly depend on the data; in fact, the goal of the training is to be able to simulate that and similar data going forward. The weights form a sort of lossy compression of the data. So the trained weights are a nonuniform representation of the data.
	
	\item \textbf{Capabilities}
	
	There's something larger going on, though a bit harder to capture mathematically. As these models train on more and more data, their capabilities grow. The language seems more natural, the models reason and problem solve better and hallucinate less, as measured by various benchmarks. The new weights, i.e., the updated nonuniformity, create an algorithm that actually computes in a different way. So the computing itself actually improves as you update the nonuniformity in the circuits.
\end{enumerate}

While the learning phase creates a nonuniform distribution from a potentially nonuniform distribution, the learning process itself is typically uniform. 

\subsection{Max-Entropy}
The KL-divergence can give a large penalty if the learned distribution puts very low weight on strings that have much larger weight on the true distribution. 

Recalling the definition of KL-divergence of $D$ from $\mu$,
\[\KL(D\|\mu)=\E{D}[\log(D(x)/\mu(x))].\]
If for some $x$, $\mu(x)\ll D(x)$ then $\KL(D\|\mu)$ can be quite high, even infinite if $D(x)>0$ while $\mu(x)=0$. 

To avoid this issue, we would like to give our learned distribution a minimum weight, equivalently a bounded max-entropy. Typical learning algorithms cooperate with us, probably because they try to minimize KL-divergence.

Recall the max-entropy of a distribution $\mu$ is
\[\max_x(\log(1/\mu(x)))\]
So a distribution $\mu$ has polynomially-bounded max-entropy if there is some $k$ such that $\mu(x)\geq 2^{-n^k}$ for all $x\in\{0,1\}^n$ which is the requirement for Theorem~\ref{prgthm}. 

Neural nets, in the abstract, never give zero weight to any output. Typically a neural net will output token $i$ with probability
\[p_i = \frac{e^{z_i}}{\sum_{j=1}^{V} e^{z_j}}\]
The $z_i$ are called \emph{logits} computed by the function
\[z_i=w_i\cdot h+b_i\]
where $h$ is a vector representation of the input computed by the neural net, $w_i$ are vectors of additional learned weights and $b_i$ is a bias scalar. If all these values are bounded by a polynomial, as they typically are in neural nets, then the output distribution has polynomially-bounded max-entropy. This even holds if there are exponentially many output tokens.

\section{Complexity as Randomness}
\label{randomsec}

As we showed in Section~\ref{distsec}, a trained machine learning model generates a computable distribution. Informally, our world creates distributions through a series of complex interactions of random and unpredictable events. Restricting to computable distributions forces the machines to treat this complexity as random events.

The main result of this section is that, from the learner's perspective, pseudorandomness and true randomness are interchangeable. Any distribution $D$ that can be decomposed into a ``structured'' part and a ``noise'' part yields the same learning objective regardless of whether the noise is uniform or pseudorandom, provided the two are computationally indistinguishable. This gives a minimum-description-length interpretation of machine learning: the model absorbs the structure it can, and treats everything that looks random to polynomial-size circuits as random.

\begin{theorem}
	\label{decompthm}
	Let $D$ and $D'$ be distributions on $\{0,1\}^n$ that are computationally indistinguishable: for every $\P/\poly$ algorithm $A$,
	\[
	\left|\Pr_{x\sim D}[A(x)=1] - \Pr_{x\sim D'}[A(x)=1]\right| \leq \negl(n).
	\]
	Then for every $\P/\poly$-computable distribution $\mu$ on $\{0,1\}^n$ with polynomially-bounded max-entropy,
	\[
	\left|\E{D}[\log(1/\mu(x))] - \E{D'}[\log(1/\mu(x))]\right| \leq \negl(n).
	\]
\end{theorem}

\begin{proof}
	Let $g(x) = \log(1/\mu(x))$. Since $\mu$ is $\P/\poly$-computable, so is $g$, to any polynomial precision. Since $\mu$ has polynomially-bounded max-entropy, there is some $c$ such that $0 \leq g(x) \leq n^c$ for all $x \in \{0,1\}^n$.
	
	Suppose for contradiction that for some polynomial $p$,
	\[
	\left|\E{D}[g(x)] - \E{D'}[g(x)]\right| > 1/p(n)
	\]
	for infinitely many $n$. Consider the $\P/\poly$ algorithm $A$ that on input $x$ outputs $1$ with probability $g(x)/n^c \in [0,1]$. (Since $g(x)$ can be irrational, $A$ actually uses an $n$-bit approximation of $g(x)/n^c$, which changes each probability below by at most $2^{-n}$ and does not affect the argument.) Then
	\[
	\left|\Pr_{x\sim D}[A(x)=1] - \Pr_{x\sim D'}[A(x)=1]\right|
	= \frac{1}{n^c}\left|\E{D}[g(x)] - \E{D'}[g(x)]\right|
	> \frac{1}{n^c p(n)},
	\]
	which is non-negligible, contradicting the indistinguishability of $D$ and $D'$.
\end{proof}

The interesting case is when $D$ is generated with real randomness and $D'$ with pseudorandomness.

\begin{corollary}
	\label{decompcor}
	Let $\nu$ be a $\P/\poly$-sampleable distribution on $\{0,1\}^{n_1}$, let $\phi:\{0,1\}^{n_1}\times\{0,1\}^{n_2}\to\{0,1\}^n$ be $\P/\poly$-computable, and let $G:\{0,1\}^k\to\{0,1\}^{n_2}$ be a cryptographic pseudorandom generator, where $n_1\leq\poly(n)$ and $n^{\Omega(1)}\leq n_2\leq\poly(n)$. Let $X\sim\nu$, let $R$ be uniform on $\{0,1\}^{n_2}$, and let $s$ be uniform on $\{0,1\}^k$, with $X$, $R$, $s$ mutually independent. Set
	\[
	D = \phi(X,R) \qquad\text{and}\qquad D' = \phi(X,G(s)).
	\]
	Then for every $\P/\poly$-computable distribution $\mu$ with polynomially-bounded max-entropy,
	\[
	\left|\E{D}[\log(1/\mu(x))] - \E{D'}[\log(1/\mu(x))]\right| \leq \negl(n).
	\]
\end{corollary}

\begin{proof}
	By security of the PRG, $R$ is computationally indistinguishable from $G(s)$. Since $X$ is independent of both, the joint distributions $(X,R)$ and $(X,G(s))$ are also computationally indistinguishable (any $\P/\poly$ distinguisher between them, given $X$ as advice or sampled internally from $\nu$, breaks $G$). Applying the $\P/\poly$-computable function $\phi$ preserves indistinguishability, so $D$ and $D'$ are computationally indistinguishable. The corollary follows from Theorem~\ref{decompthm}.
\end{proof}

Corollary~\ref{decompcor} has a minimum-description-length reading. The cross-entropy $\E{D}[\log(1/\mu(x))]$ is the expected description length of a sample from $D$ under the code induced by $\mu$. The corollary says this quantity is insensitive to replacing the true noise $R$ with pseudorandom noise $G(s)$. A learner minimizing cross-entropy therefore cannot tell, and has no reason to try to tell, whether the randomness it sees is real. It commits description length to the structure $\nu$ and $\phi$ that it can model, and treats the residual $n_2$ bits of noise as incompressible, whether they genuinely are or merely look that way to polynomial-size circuits. This is the precise sense in which machine learning manages complexity through probability: pseudorandomness is, for learning purposes, randomness.

Theorem~\ref{prgthm} in the rest of this section is the extreme case of Corollary~\ref{decompcor} in which the structure is trivial: $\nu$ is empty, $\phi$ is the identity, so $D$ is uniform and $D'$ is the output of $G$, and all the apparent complexity is pseudorandom. In that case, the optimal $\mu$ treats the generator's output $D'$ as uniform.

\begin{theorem}
	\label{prgthm}
	Let $D$ be the output distribution of a cryptographic pseudorandom generator on $\{0,1\}^n$, and let $U$ denote the uniform distribution on $\{0,1\}^n$. Suppose $\mu$ is a $\P/\poly$-computable distribution with polynomially-bounded max-entropy such that $\KL(D\|\mu) \leq \KL(D\|U)$. Then $\KL(U\|\mu)$ is negligible in $n$.
\end{theorem}

\begin{proof}
	Since $D$ is the output of a cryptographic PRG, $D$ and $U$ are computationally indistinguishable. By Theorem~\ref{decompthm},
	\[
	\E{U}[\log(1/\mu(x))] \leq \E{D}[\log(1/\mu(x))] + \negl(n).
	\]
	The hypothesis $\KL(D\|\mu) \leq \KL(D\|U)$ rearranges to $\E{D}[\log(1/\mu(x))] \leq \E{D}[\log(1/U(x))] = n$. Hence
	\[
	\KL(U\|\mu) = \E{U}[\log(1/\mu(x))] - n \leq \negl(n).\qedhere
	\]
\end{proof}

We need some sort of max-entropy bound for Theorem~\ref{prgthm}, because if $\mu(x)$ was extremely small for some $x$, $\KL(U\|\mu)$ could be very high. However we can use the proof of Theorem~\ref{prgthm} to show that even if we remove the max-entropy condition, we still get $\mu$ statistically close to uniform. First we need the following lemma.

\begin{lemma}
	\label{sdlem}
Let $D$ and $\mu$ be distributions over $\{0,1\}^n$. Define
	$\mu'(x) = (1-2^{-n})\mu(x) + 2^{-2n}$
	for all $x\in\{0,1\}^n$. Then \[\KL(D\|\mu') \le \KL(D\|\mu) + \log\!\left(\frac{1}{1-2^{-n}}\right)\leq\KL(D\|\mu) + \negl(n).
	\]
\end{lemma}

\begin{proof}
	First observe that $\mu'$ is a probability distribution. Indeed,
	\[
	\sum_{x\in\{0,1\}^n} \mu'(x)
	= (1-2^{-n})\sum_x \mu(x) + \sum_x 2^{-2n}
	= (1-2^{-n}) + 2^n\cdot 2^{-2n}
	= 1.
	\]
	
	For every $x$, we have $\mu'(x) = (1-2^{-n})\mu(x)+2^{-2n} \ge (1-2^{-n})\mu(x)$. Hence for any $x$ with $D(x)>0$,
	\[
	\log \frac{D(x)}{\mu'(x)}
	\le
	\log \frac{D(x)}{(1-2^{-n})\mu(x)}
	=
	\log \frac{D(x)}{\mu(x)} + \log\!\left(\frac{1}{1-2^{-n}}\right).
	\]
	
	Multiplying by $D(x)$ and summing over $x$ gives $\KL(D\|\mu') \le \KL(D\|\mu) + \log\!\left(\frac{1}{1-2^{-n}}\right)$. Finally, $\log\!\left(\frac{1}{1-2^{-n}}\right)=O(2^{-n})$, which is negligible.
\end{proof}

If we do not restrict the max-entropy of $\mu$, we can still show that any $\P/\poly$-computable $\mu$ that models $D$ at least as well as $U$ does will be statistically close to uniform.

\begin{corollary}
	\label{sdprgcor}
		Let $D$ be the output distribution of a cryptographic pseudorandom generator on
	$\{0,1\}^n$, and let $U$ denote the uniform distribution on $\{0,1\}^n$. 
	
	If $\mu$ is a $\P/\poly$-computable distribution such that \[\KL(D\|\mu)\leq \KL(D\|U),\] then the statistical difference between $U$ and $\mu$
	is negligible in $n$.
\end{corollary}

\begin{proof}
	Let $\mu'$ be the distribution from Lemma~\ref{sdlem} with $\KL(D\|\mu') \leq\KL(D\|\mu) + \negl(n)$. $\mu'$ is a $\P/\poly$-computable distribution with polynomially-bounded max-entropy. 
	
	We have $\KL(D\|\mu)\leq\KL(D\|U)$ so $\KL(D\|\mu')\leq\KL(D\|U)+\negl(n)$. By the proof of Theorem~\ref{prgthm}, we have $\KL(U\|\mu')\leq\negl(n)$. 
	
	By Pinsker's inequality (see~\cite{CoverThomas}), the statistical difference of $U$ and $\mu'$ is bounded by $\sqrt{\frac{\ln 2}{2}\,\KL(U\|\mu')}$ which is $\negl(n)$ and since $\mu$ and $\mu'$ have negligible statistical difference, the statistical difference of $U$ and $\mu$ is negligible.
\end{proof}

In Appendix~\ref{pesec} we show that Theorem~\ref{prgthm} is a special case of a general phenomenon. Using a characterization of pseudoentropy by Vadhan and Zheng~\cite{VZ12}, the best cross-entropy achievable by an efficient model is exactly the next-bit pseudoentropy of the data. The model absorbs the structure that efficient observers can predict and treats the rest as random.

\section{Conclusions}
\label{togethersec}

In Section~\ref{distsec}, we showed that we can model modern machine learning models as $\P/\poly$-computable distributions with polynomially-bounded max-entropy. The nonuniformity and full-depth polynomial circuits allow the flexibility to learn from complex distributions.

In Section~\ref{klsec}, we showed that minimizing the \KL-divergence over $\P/\poly$-computable distributions best captures the complexity of describing the distribution of strings, consistent with the distributions from Section~\ref{distsec}.

Instead of rigidly forcing the hypothesis to give us a single answer, by looking at distributions we allow the models to make mistakes, and with polynomially-bounded max-entropy the models never completely eliminate any possible output. This allows the modeling of complex behavior by probability, which we illustrate by Theorem~\ref{prgthm} in Section~\ref{randomsec} and Corollary~\ref{prgcecor}, which show that we can't do better at modeling the output of a pseudorandom generator than by a uniform distribution. 

James Kurose, the noted networking researcher, has a saying that the Internet worked so well because it doesn't have to. In other words, because we don't require Internet protocols to guarantee delivery, we can create protocols that are far more robust and powerful. In this paper, we show a similar property for machine learning: because we don't require complete accuracy, the models can capture far more complex behavior. 
	
\section{Future Work}
\label{futuresec}

In a 2022 CACM article~\cite{pvnp50-O}, the author observed that many of the good implications from $\P=\NP$ are coming true while cryptography remains secure, a world he called ``Optiland,'' the supposedly impossible counterpart to Russell Impagliazzo's ``Pessiland''~\cite{Impagliazzo95}. This paper is a first step toward understanding how machine learning makes Optiland possible from a computational complexity viewpoint. Much remains to be done.

\subsection{Beyond pseudorandom generators}

Theorem~\ref{prgthm} tells us what the optimal $\P/\poly$-computable $\mu$ looks like when $D$ is the output of a cryptographic pseudorandom generator: information-theoretically close to uniform. The same argument should say something more general. For any $\P/\poly$-sampleable distribution $D$, what is the min-KL $\P/\poly$-computable $\mu$ with polynomially-bounded max-entropy? 

Appendix~\ref{pesec} answers part of this question: the best achievable cross-entropy is the next-bit pseudoentropy of $D$~\cite{HRV13,VZ12}, which can exceed HILL pseudoentropy and depends on the order of the bits (Proposition~\ref{orderprop}). When $D$ is pseudorandom against $\P/\poly$, that reduces to Theorem~\ref{prgthm}; when $D$ is itself computable, $\mu=D$ works. What remains open is what the optimal $\mu$ itself looks like. Is every near-optimal model close to the distribution given by the pseudoentropy witnesses, as Theorem~\ref{prgthm} shows for pseudorandom generators?

\subsection{The complexity of learning}
\label{learningsec}

This paper models the \emph{output} of machine learning, not the learning process itself. Theorem~\ref{prgthm} says what a good model looks like, not how to find one: any $\P/\poly$-computable $\mu$ that models $D$ at least as well as $U$ does must be close to uniform. The question of whether an efficient learner can actually find a $\P/\poly$-computable $\mu$ minimizing $\KL(D\|\mu)$ for a general sampleable $D$ is different, and the answer is sometimes no.

Consider a pseudorandom permutation $f_k$, like AES with a secret key $k$, as in the hard distributions of Section~\ref{kmrsec}. Then $(x,f_k(x))$ for random $x$ is $\P/\poly$-sampleable and $\P/\poly$-computable with the key as advice. An efficient learner that found this computable distribution from samples could predict $f_k$ on new inputs, breaking the pseudorandomness of $f_k$. So the min-KL $\P/\poly$-computable distribution is not in general efficiently findable.

This suggests a formal notion. Say a distribution $D$ is \emph{efficiently $\P/\poly$-learnable} if there is a probabilistic polynomial-time algorithm that, given sample access to $D$, outputs a circuit computing a $\P/\poly$-computable $\mu$ with $\KL(D\|\mu)$ close to the minimum achievable. Which sampleable distributions are efficiently learnable? The Hirahara-Nanashima results~\cite{HN-Pessiland} give one piece of this: in a world without one-way functions, more distributions become learnable. Cavalar, Goldin, Gray, Hiroka, Hsieh and Morimae~\cite{CGGHHM25} give a similar equivalence for verifying rather than learning: every sampleable distribution can be efficiently verified exactly when one-way functions do not exist. Characterizing the full landscape, and relating efficient learnability to the cryptographic hierarchy, would sharpen our picture of what modern machine learning can and cannot do.

\subsection{Meta-complexity}

The KL-minimization problem has a meta-complexity flavor that this paper only begins to exploit. The cross-entropy term
\[
\E{D}[\log 1/\mu(x)]
\]
is the expected description length of $x\sim D$ under the code induced by $\mu$. Minimizing $\KL(D\|\mu)$ over $\P/\poly$-computable $\mu$ is therefore closely related to finding optimal $\P/\poly$-bounded codes for $D$, which is in turn closely related to time-bounded Kolmogorov complexity $\K^{\poly}$. The Antunes-Fortnow-van Melkebeek-Vinodchandran characterizations~\cite{depth-J} of $\P$ and average-case complexity in terms of $\K^{\poly}$ feel like the right language. Appendix~\ref{codingapp} takes a first step, showing that the expected time-bounded Kolmogorov complexity of samples and the best efficient cross-entropy can differ by nearly $n$ bits. Making the full connection precise might give us a meta-complexity version of Theorem~\ref{prgthm} and its generalization above.

\subsection{Refining the abstraction}

Modeling machine learning as $\P/\poly$-computable with polynomially-bounded max-entropy is a useful abstraction but a loose upper bound. A transformer evaluated on a single forward pass is closer to $\TC^0$ than to $\P/\poly$; polynomial time emerges only through chain-of-thought reasoning or code execution, which effectively adds an unbounded number of $\TC^0$ layers with intermediate memory. A tighter model would distinguish between what the network computes per token (bounded depth) and what reasoning traces achieve (polynomial depth). Results like those of Merrill and Sabharwal~\cite{MerrillSabharwal} on the expressive power of transformers with intermediate steps are relevant here. The results of Section~\ref{randomsec} and the lower bound of Theorem~\ref{nbpelower} hold automatically for any refined class contained in $\P/\poly$, though the upper bound of Theorem~\ref{nbpeupper} need not. The open question is whether such a class can still come close to the best loss that $\P/\poly$ achieves, and whether it captures something meaningful about real models that $\P/\poly$ does not.

\subsection{Toward Optiland}

Finally, the paper is positioned as a first step in explaining Optiland, but the full story requires going from ``machine learning models complex distributions'' to ``machine learning solves hard optimization problems.'' The bridge is average-case complexity: real-world instances of $\NP$-hard problems come from distributions that may themselves be $\P/\poly$-sampleable, and if machine learning can capture these distributions well enough, it may be able to exploit structure that worst-case complexity hides. Connecting the framework of this paper to average-case hardness, and to the Heuristica-like regime where natural distributions are easy even if worst-case problems are hard, is the larger program this work hopes to open.

\section*{AI Methodology Statement}

Most of this paper was written by the author. The author had discussions on these topics with several AI models including Gemini, ChatGPT and Claude. ChatGPT helped with the proof and write-up of Theorem~\ref{prgthm} and Theorem~\ref{martthm}. Claude (Anthropic) suggested the connection to the work of Vadhan and Zheng, drafted Appendices~\ref{codingapp} and~\ref{pesec}, including the statements and proofs of the results there, drafted Sections~\ref{pacsec} and~\ref{kmrsec}, drafted several new paragraphs based on the author's ideas, including most of Section~\ref{frozensec}, and helped redraft others throughout the paper. Claude and Gemini were also used for proofreading, stylistic suggestions and literature search. The author verified and takes full responsibility for all proofs, text and references.

\section*{Acknowledgments}

The author thanks Matthew Gray and Taiga Hiroka for helpful discussions and for sharing ideas including from their then-forthcoming paper with Bruno Cavalar, Eli Goldin, Min-Hsiu Hsieh and Tomoyuki Morimae~\cite{CGGHHM25}. These ideas helped formulate the statement of Theorem~\ref{prgthm}, the counterexample in Section~\ref{learningsec} that we couldn't fully minimize the $\KL$-divergence of a $\P/\poly$-computable distribution without breaking cryptography, and the general idea that the learning process treats pseudorandomness like randomness (Section~\ref{klsec}).

The author also thanks Rohan Acharya, Harry Buhrman, Evangelos Georgiadis, Om Mishra, Rahul Santhanam and Rakesh Vohra for helpful discussions, as well as countless others the author has talked to about these questions over the past several years.

\appendix

\section{Computable Distributions and Conditional Probability Functions}
\label{martsec}

\begin{theorem}
	\label{martthm}
	Let $\mu$ be a distribution on $\{0,1\}^n$. The following are equivalent.

	\begin{enumerate}
		\item The cumulative distribution function
		\[
		F_\mu(x)=\sum_{z \le x} \mu(z)
		\]
		is computable in polynomial time, where $\le$ is the lexicographic order on
		$\{0,1\}^n$.

		\item The associated conditional probability function
		\[
		f_\mu(y)=\Pr_{x\sim\mu}[x_{|y|+1}=1 \mid x_{1\ldots |y|}=y]
		\]
		is in $\mathrm{FP}$ for all prefixes $y\in\{0,1\}^{<n}$, with the convention
		that $f_\mu(y)=0$ if $\mu(y0)+\mu(y1)=0$.
	\end{enumerate}
	Here values are rationals represented by numerator and denominator.
\end{theorem}

\begin{proof}
	For a prefix $y\in\{0,1\}^k$, let
	\[
	\mu(y)=\Pr_{x\sim\mu}[x_{1\ldots k}=y]
	\]
	denote the cylinder probability of $y$. Then
	\[
	f_\mu(y)=\frac{\mu(y1)}{\mu(y)}
	\]
	whenever $\mu(y)>0$, and $f_\mu(y)=0$ otherwise.

	We show each direction separately.

	\medskip

	\noindent\textbf{(1) $\Rightarrow$ (2).}
	Assume $F_\mu\in \mathrm{FP}$. We first show that for every prefix $y$,
	the quantity $\mu(y)$ is polynomial-time computable.

	Let $|y|=k$. The set of strings in $\{0,1\}^n$ having prefix $y$ is a
	contiguous interval in lexicographic order, namely
	\[
	\{y0^{n-k},\, y0^{n-k}+1,\,\dotsc,\, y1^{n-k}\}.
	\]
	Thus if we write
	\[
	a_y = y0^{\,n-k}
	\qquad\text{and}\qquad
	b_y = y1^{\,n-k},
	\]
	then
	\[
	\mu(y)=\sum_{a_y \le z \le b_y} \mu(z).
	\]
	Using the CDF, this becomes
	\[
	\mu(y)=F_\mu(b_y)-F_\mu(a_y-1),
	\]
	where $a_y-1$ denotes the lexicographic predecessor of $a_y$, and if
	$a_y$ is the minimum string $0^n$, then $F_\mu(a_y-1)$ is understood to be $0$.

	Similarly,
	\[
	\mu(y1)=F_\mu(b_{y1})-F_\mu(a_{y1}-1),
	\]
	so both $\mu(y)$ and $\mu(y1)$ are computable in polynomial time.

	Therefore
	\[
	f_\mu(y)=
	\begin{cases}
		\mu(y1)/\mu(y), & \text{if } \mu(y)>0,\\[1ex]
		0, & \text{if } \mu(y)=0,
	\end{cases}
	\]
	is polynomial-time computable as well. Hence $f_\mu\in \mathrm{FP}$.

	\medskip

	\noindent\textbf{(2) $\Rightarrow$ (1).}
	Assume $f_\mu\in \mathrm{FP}$. We show that $F_\mu\in \mathrm{FP}$.

	First observe that from $f_\mu$ one can compute $\mu(x)$ for every
	$x=x_1x_2\cdots x_n\in\{0,1\}^n$. Indeed, if $x_{<i}=x_1\cdots x_{i-1}$,
	then
	\[
	\mu(x)
	=
	\prod_{i=1}^n
	\Pr[x_i \mid x_{<i}]
	=
	\prod_{i=1}^n
	\begin{cases}
		f_\mu(x_{<i}), & \text{if } x_i=1,\\[1ex]
		1-f_\mu(x_{<i}), & \text{if } x_i=0.
	\end{cases}
	\]
	This is polynomial-time computable since there are only $n$ factors and
	each value $f_\mu(x_{<i})$ is polynomial-time computable.

	Now let $x\in\{0,1\}^n$. We express $F_\mu(x)=\sum_{z\le x}\mu(z)$ as a sum of
	cylinder probabilities. Suppose
	\[
	x = x_1x_2\cdots x_n.
	\]
	A string $z\le x$ either equals $x$, or is determined by the first position $i$ at which $z_i<x_i$.
	The latter can happen only when $x_i=1$, in which case $z_i=0$ and the previous
	bits agree with $x$. Hence
	\[
	F_\mu(x)
	=
	\mu(x)+\sum_{i:\,x_i=1} \mu(x_1x_2\cdots x_{i-1}0).
	\]

	So it suffices to compute each cylinder probability $\mu(y)$ in polynomial time.
	If $y=y_1\cdots y_k$, then
	\[
	\mu(y)
	=
	\prod_{i=1}^k
	\begin{cases}
		f_\mu(y_{<i}), & \text{if } y_i=1,\\[1ex]
		1-f_\mu(y_{<i}), & \text{if } y_i=0.
	\end{cases}
	\]

	Thus $\mu(x)$ and every term $\mu(x_1\cdots x_{i-1}0)$ are polynomial-time computable,
	and there are at most $n+1$ such terms. Therefore $F_\mu(x)$ is computable
	in polynomial time, and hence $F_\mu\in \mathrm{FP}$.
\end{proof}

\section{Compression and Description Length}
\label{codingapp}

This appendix makes the connection in Section~\ref{klsec} between computable distributions, description length and the training objective of neural nets precise, and shows a limit of the Occam's razor intuition. We work with an explicit form of the models of Section~\ref{distsec}. Building compressors from next-bit predictors by arithmetic coding is standard~\cite{CoverThomas,DRD+24,HMS23,HLN24}.

\subsection{Next-Bit Predictors}

For $p\in(0,1)$ and $b\in\{0,1\}$ define the \emph{log loss} $\ell(p,b)=\log(1/p)$ if $b=1$ and $\ell(p,b)=\log(1/(1-p))$ if $b=0$.

\begin{definition}
\label{preddef}
A \emph{next-bit predictor} on $\{0,1\}^n$ with precision $m$ is a circuit $C$ that on input a string $y\in\{0,1\}^{<n}$ (under a standard encoding of strings of length less than $n$) outputs a number $C(y)\in\{1,2,\dots,2^m-1\}\cdot 2^{-m}$. It defines the distribution
\[
\mu_C(x)=\prod_{i=1}^n 2^{-\ell(C(x_{<i}),x_i)},
\]
that is, $C(y)$ is the probability under $\mu_C$ that the bit after prefix $y$ is $1$. We let $\Pred_{s,m}$ denote the set of next-bit predictors on $\{0,1\}^n$ with precision $m$ computed by circuits of size at most $s$.
\end{definition}

Every $\mu_C(x)$ is a dyadic rational with at most $nm$ bits and $\mu_C(x)\geq 2^{-nm}$, so $\mu_C$ has max-entropy at most $nm$. By Theorem~\ref{martthm} in Appendix~\ref{martsec}, a family of predictors in $\Pred_{s,m}$ with $s$ and $m$ polynomial in $n$ defines a $\P/\poly$-computable distribution with polynomially-bounded max-entropy.

\begin{definition}
For a distribution $D$ on $\{0,1\}^n$, the \emph{efficient cross-entropy} of $D$ is
\[
\CE_{s,m}(D)=\min_{C\in\Pred_{s,m}}\E{D}\left[\log\frac{1}{\mu_C(x)}\right]=\H(D)+\min_{C\in\Pred_{s,m}}\KL(D\|\mu_C).
\]
\end{definition}

The minimum exists since $\Pred_{s,m}$ is finite. Always $\H(D)\leq\CE_{s,m}(D)\leq n$ for $s$ large enough to output the constant $1/2$.

\subsection{Prefix-Free Kolmogorov Complexity}

Fix a universal prefix-free Turing machine $U$. For strings $x$ and $y$ and a time bound $t$, the time-bounded Kolmogorov complexity $\K^t(x\mid y)$ is the length of the shortest program $p$ such that $U(p,y)$ outputs $x$ within $t$ steps, and $\K^t(x)$ is $\K^t(x\mid y)$ for $y$ the empty string. Because the programs form a prefix-free set, Kraft's inequality gives $\sum_x 2^{-\K^t(x\mid y)}\leq 1$, and so for every distribution $D$, every $y$ and every $t$,
\begin{equation}
\label{kraftlower}
\E{D}[\K^t(x\mid y)]\geq\H(D).
\end{equation}
We write $\enc{C}$ for a standard self-delimiting description of a circuit $C$.

\subsection{Compression with Next-Bit Predictors}

\begin{theorem}
\label{codingthm}
There is a polynomial $q$ such that for every next-bit predictor $C\in\Pred_{s,m}$ on $\{0,1\}^n$ there is a compression function $e_C:\{0,1\}^n\to\{0,1\}^*$, whose outputs form a prefix-free set, with the following properties.
\begin{enumerate}
\item For every $x\in\{0,1\}^n$, \[\log\frac{1}{\mu_C(x)}+1\;\leq\;|e_C(x)|\;\leq\;\log\frac{1}{\mu_C(x)}+2\log(nm)+O(1).\]
\item There is a fixed algorithm that, given $\enc{C}$ and $e_C(x)$, outputs $x$ in time $q(n,s,m)$, and a fixed algorithm that computes $e_C(x)$ from $\enc{C}$ and $x$ in time $q(n,s,m)$.
\item Consequently, for $t\geq q(n,s,m)$, $\K^t(x\mid\enc{C})\leq\log(1/\mu_C(x))+2\log(nm)+O(1)$, and for some fixed polynomial $q'$ and every $t'$, $\K^{q'(t',n,s,m)}(x)\leq\K^{t'}(\enc{C})+\log(1/\mu_C(x))+2\log(nm)+O(1)$.
\end{enumerate}
\end{theorem}

\begin{proof}
Write $\mu=\mu_C$, and for a prefix $y$ let $\mu(y)=\prod_{i\leq|y|}2^{-\ell(C(y_{<i}),y_i)}$ be the probability under $\mu$ that a sample begins with $y$. Let $F(x)=\sum_{z<x}\mu(z)$, where $<$ is the lexicographic order on $\{0,1\}^n$. A string $z<x$ is determined by the first position $i$ where it differs from $x$, at which $x_i=1$ and $z_i=0$, so $F(x)=\sum_{i:x_i=1}\mu(x_{<i}0)$, a sum of at most $n$ products of at most $n$ outputs of $C$. For each prefix $y$ let $I_y=[a_y,a_y+\mu(y))$ where $a_y=F(y0^{n-|y|})$. The intervals $I_x$ for $x\in\{0,1\}^n$ partition $[0,1)$, and $I_y$ is the disjoint union of $I_{y0}=[a_y,a_y+\mu(y0))$ and $I_{y1}=[a_y+\mu(y0),a_y+\mu(y))$.

Let $k=\lceil\log(1/\mu(x))\rceil+1$, so $2^{-k}\leq\mu(x)/2$, and let $w=\lceil 2^kF(x)\rceil$ and $v=w/2^k$. Then $F(x)\leq v< F(x)+2^{-k}<F(x)+\mu(x)$, so $v\in I_x$. Let $e_C(x)$ be a self-delimiting encoding of $k$, using $2\lceil\log(k+1)\rceil+2$ bits, followed by the $k$-bit binary representation of $w$ (note $w<2^k$ since $v<1$). The compressed strings form a prefix-free set since $k$ determines the length of the rest. Since $\mu(x)\geq 2^{-nm}$ we have $k\leq nm+2$, which gives the bounds in part~1.

To decompress, read $k$ and $w$, set $v=w/2^k$, and maintain a prefix $y$ (initially empty), the value $a=a_y$ (initially $0$) and $\mu(y)$ (initially $1$). For $n$ steps, compute $\mu(y0)=\mu(y)(1-C(y))$; if $v<a+\mu(y0)$ replace $y$ by $y0$, and otherwise replace $a$ by $a+\mu(y0)$ and $y$ by $y1$, updating $\mu(y)$ accordingly. Since $v\in I_x\subseteq I_y$ whenever $y$ is a prefix of $x$, the choice made at each step agrees with the next bit of $x$, so the decompressor outputs $x$. Each step evaluates $C$ once and does exact arithmetic on dyadic rationals of at most $nm+k$ bits, giving part~2. Compression similarly computes $\mu(x)$ and $F(x)$ exactly.

For part~3, a program for $x$ given $\enc{C}$ consists of a constant-size decompressor followed by $e_C(x)$. For the unconditional bound, the program first runs a shortest program producing $\enc{C}$ within $t'$ steps, which is self-delimiting, and then decompresses.
\end{proof}

Taking expectations, for every $D$ and every $C\in\Pred_{s,m}$,
\begin{equation}
\label{mdlbound}
\E{D}[\K^{q'(t',n,s,m)}(x)]\;\leq\;\K^{t'}(\enc{C})+\H(D)+\KL(D\|\mu_C)+2\log(nm)+O(1).
\end{equation}
The right-hand side is a two-part minimum description length: the cost of describing the model plus the expected cost of describing a sample given the model.

\subsection{The Next-Token Identity}

Neural nets are not trained to minimize compressed length directly. They are trained to minimize the cross-entropy of their next-token predictions, summed over positions. The following identity, which is immediate from the definition of $\mu_C$, shows that this is the same objective.

\begin{proposition}
\label{nexttokenprop}
For every distribution $D$ on $\{0,1\}^n$ and every next-bit predictor $C$,
\[
\E{D}\left[\log\frac{1}{\mu_C(x)}\right]=\sum_{i=1}^n\E{D}\big[\ell(C(x_{<i}),x_i)\big]=\H(D)+\KL(D\|\mu_C).
\]
Moreover $\KL(D\|\mu_C)=\sum_{i=1}^n\E{D}\left[\KL\big(D(\cdot\mid x_{<i})\,\|\,\mathrm{Bern}(C(x_{<i}))\big)\right]$, where $D(\cdot\mid x_{<i})$ is the conditional distribution of the $i$-th bit given the prefix.
\end{proposition}

\begin{proof}
The first equality takes the logarithm of the product defining $\mu_C$ and uses linearity of expectation. The second is the cross-entropy decomposition. The last statement is the chain rule for KL divergence, or equivalently follows by subtracting the chain rule $\H(D)=\sum_i\H(X_i\mid X_{<i})$ term by term.
\end{proof}

The middle expression is exactly the training loss of a language model: at each position, the negative log of the probability the model assigned to the token that actually occurred, averaged over the training data, and summed over positions. Language models predict tokens from a vocabulary of size $V$ rather than bits, but a token is a block of $\lceil\log V\rceil$ bits, and the per-token loss $\log 1/P(\text{token}\mid\text{context})$ is exactly the sum of the per-bit losses of the induced bit-level conditionals, by the same chain rule. When $V$ is polynomial, the bit-level conditionals are computable from the token probabilities by summing over at most $V$ tokens, so a polynomial-size token predictor gives a polynomial-size next-bit predictor, up to rounding to $m$ bits of precision.

Combining Theorem~\ref{codingthm} and Proposition~\ref{nexttokenprop}, for every $C\in\Pred_{s,m}$,
\[
\E{D}[|e_C(x)|]-O(\log nm)\;\leq\;\sum_{i=1}^n\E{D}\big[\ell(C(x_{<i}),x_i)\big]\;=\;\H(D)+\KL(D\|\mu_C)\;\leq\;\E{D}[|e_C(x)|]-1,
\]
so minimizing the expected length of model-based compression, minimizing $\KL(D\|\mu_C)$, and minimizing the summed next-token loss over any class of predictors give the same minimizers up to an additive $O(\log nm)$, and the minimum value of all three is $\CE_{s,m}(D)$ up to the same error.

Let's connect this with neural nets. The term $\K^{t'}(\enc{C})$ in~\eqref{mdlbound} is dominated by the number of weights in the model, a hyperparameter to be optimized. Fixing that size, training finds the weights that minimize the summed next-token loss, which is the expected description length of a sample under the model. This is the precise sense in which neural-net training is a minimum description length procedure, restricted to descriptions of the model-based form.

\subsection{Model-Based Compression Does Not Find Shortest Descriptions}

One might hope, following an Occam's razor argument, that minimizing model-based compressed length finds the best explanation of $D$ in the sense of shortest efficient descriptions. It does not. By~\eqref{kraftlower}, no description can do better than $\H(D)$ on average. For pseudorandom distributions, time-bounded Kolmogorov complexity needs only about the seed length, while every model-based compression stays far away.

\begin{theorem}
\label{septhm}
Let $G:\{0,1\}^\ell\to\{0,1\}^n$ be computed by a circuit of size $s_G$ and be $(s',\epsilon)$-pseudorandom, and let $D=G(U_\ell)$. There is a constant $c$ and a polynomial $q$ such that for $t\geq q(s_G)$:
\begin{enumerate}
\item For every $x$ in the support of $D$, $\K^t(x\mid\enc{G})\leq\ell+O(1)$.
\item For every $C\in\Pred_{s,m}$ with $s+c(m+n)\leq s'$, $\E{D}[|e_C(x)|]\geq\E{D}[\log 1/\mu_C(x)]\geq n-2nm\epsilon$.
\end{enumerate}
\end{theorem}

\begin{proof}
For part~1, given $\enc{G}$, a constant-size program followed by a seed $r$ with $G(r)=x$ outputs $x$ in time polynomial in $s_G$, and the program is prefix-free because $\ell$ is determined by $\enc{G}$. Part~2 follows from Theorem~\ref{codingthm} and Corollary~\ref{prgcecor} in Appendix~\ref{pesec}.
\end{proof}

So for a secure generator, with $\epsilon$ negligible and $m$ polynomial, model-based compression is longer than the shortest efficient description by at least $n-\ell-O(1)-\negl(n)$ bits, even though model-based compression is given its model $C$ for free and $C$ may be far larger than $G$. The shortest efficient descriptions see the seed. The model sees only what it can predict bit by bit, and the rest looks random to it. Appendix~\ref{pesec} shows that this is the general picture: the best model-based compressed length is the next-bit pseudoentropy of $D$, while no description, efficient or not, can do better than $\H(D)$ on average by~\eqref{kraftlower}. The gap is the cryptographic content of $D$. It is essentially zero when $D$ is itself computed by an efficient next-bit predictor.

\section{Efficient Cross-Entropy and Next-Bit Pseudoentropy}
\label{pesec}

Theorem~\ref{prgthm} says what an optimal model looks like when all of the complexity of $D$ is pseudorandom. In this appendix we characterize the optimal loss for every distribution $D$, using the next-bit predictors and efficient cross-entropy of Appendix~\ref{codingapp}. The answer is a notion of computational entropy from the cryptographic literature, and the main tool is a characterization of pseudoentropy by Vadhan and Zheng~\cite{VZ12}. 

These connections between next-token prediction, compression and pseudoentropy have precedents. Iftach Haitner, Noam Mazor and Jad Silbak~\cite{HMS23} combined the Vadhan--Zheng characterization with compression from next-bit predictors to show that incompressible distributions have next-block pseudoentropy, and Shuichi Hirahara, Zhenjian Lu and Mikito Nanashima~\cite{HLN24} gave optimal coding theorems from next-bit predictors. Independently and concurrently with this appendix, Noam Mazor, Andrew Morgan and Rafael Pass~\cite{MMP26} showed that next-token predictors are equivalent to monotone compression and that next-bit pseudoentropy is equivalent to monotone incompressibility, closely related to Corollary~\ref{charcor}, and Gal Arnon, Andrej Bogdanov, Paul Lou, Giulio Malavolta, Alon Rosen and Lorenzo Taschin~\cite{ABLMRT26} used the Vadhan--Zheng characterization to study autoregressive sampling.

For concrete statements, a generator $G:\{0,1\}^\ell\to\{0,1\}^n$ is \emph{$(s,\epsilon)$-pseudorandom} if every circuit of size at most $s$ distinguishes $G(U_\ell)$ from $U_n$ with advantage at most $\epsilon$, where $U_k$ denotes the uniform distribution on $\{0,1\}^k$. For distinguishing two fixed distributions, fixing the best random bits makes a randomized circuit deterministic, so we use deterministic distinguishers throughout.

\begin{definition}[\cite{HRV13,VZ12}]
\label{nbpedef}
Let $X$ be distributed according to $D$ on $\{0,1\}^n$. We say $D$ has \emph{$(s,\epsilon)$-next-bit pseudoentropy at least $k$} if there are random variables $Y_1,\dots,Y_n\in\{0,1\}$, jointly distributed with $X$, such that
\begin{enumerate}
\item $\sum_{i=1}^n\H(Y_i\mid X_{<i})\geq k$, and
\item for every circuit $A$ of size at most $s$ that takes as input an index $i\in[n]$ and a string in $\{0,1\}^i$,
\[
\Big|\Ex_{i}\Pr[A(i,X_{<i}X_i)=1]-\Ex_i\Pr[A(i,X_{<i}Y_i)=1]\Big|\leq\epsilon,
\]
where $i$ is uniform in $[n]$.
\end{enumerate}
\end{definition}

This is the formalization of Haitner, Reingold and Vadhan's notion used by Vadhan and Zheng~\cite[Definition 2.12]{VZ12}, written with one random variable $Y_i$ for each index. Since $\H(X)=\sum_i\H(X_i\mid X_{<i})$, taking $Y_i=X_i$ shows every $D$ has next-bit pseudoentropy at least $\H(D)$. The interesting case is when it is larger.

\subsection{Pseudoentropy Lower Bounds the Loss}

\begin{theorem}
\label{nbpelower}
There is a constant $c$ such that the following holds. If $D$ has $(s+c(m+n),\epsilon)$-next-bit pseudoentropy at least $k$, then every $C\in\Pred_{s,m}$ satisfies $\E{D}[\log 1/\mu_C(x)]\geq k-nm\epsilon$. In other words, $\CE_{s,m}(D)\geq k-nm\epsilon$.
\end{theorem}

\begin{proof}
Let $X\sim D$ and let $Y_1,\dots,Y_n$ witness the pseudoentropy. Write $\ell_i(z,b)=\ell(C(z),b)$ for $z\in\{0,1\}^{i-1}$, so that $0\leq\ell_i\leq m$ since the outputs of $C$ lie in $[2^{-m},1-2^{-m}]$. By Proposition~\ref{nexttokenprop},
\[
\E{D}\left[\log\frac{1}{\mu_C(x)}\right]=\sum_{i=1}^n\Ex[\ell_i(X_{<i},X_i)].
\]
For each $i$, Gibbs' inequality applied conditionally on $X_{<i}$ gives $\Ex[\ell_i(X_{<i},Y_i)]\geq\H(Y_i\mid X_{<i})$, so $\sum_i\Ex[\ell_i(X_{<i},Y_i)]\geq k$. It remains to show
\begin{equation}
\label{hybridgap}
\sum_{i=1}^n\big(\Ex[\ell_i(X_{<i},Y_i)]-\Ex[\ell_i(X_{<i},X_i)]\big)\leq nm\epsilon.
\end{equation}
For a nonnegative random variable $L$ bounded by $m$, $\Ex[L]=\int_0^m\Pr[L>\theta]\,d\theta$. So the left side of~\eqref{hybridgap} equals
\[
n\int_0^m\Big(\Ex_i\Pr[\ell_i(X_{<i},Y_i)>\theta]-\Ex_i\Pr[\ell_i(X_{<i},X_i)>\theta]\Big)d\theta.
\]
For fixed $\theta$, consider the circuit $A_\theta(i,zb)$ that outputs $1$ when $\ell(C(z),b)>\theta$. The event $\ell(C(z),b)>\theta$ says that the probability $C$ assigns to $b$ is less than $2^{-\theta}$, and since that probability is a multiple of $2^{-m}$ this is a comparison of an $m$-bit number against a fixed $m$-bit threshold. So $A_\theta$ has size at most $s+c(m+n)$ for a suitable constant $c$, where the $n$ term covers handling the index and the encoding of prefixes. By the pseudoentropy hypothesis the integrand is at most $\epsilon$ in absolute value for every $\theta$, giving~\eqref{hybridgap}. Hence $\E{D}[\log 1/\mu_C(x)]\geq k-nm\epsilon$.
\end{proof}

\begin{corollary}
\label{prgcecor}
Let $G:\{0,1\}^\ell\to\{0,1\}^n$ be $(s',\epsilon)$-pseudorandom and $D=G(U_\ell)$. For every $C\in\Pred_{s,m}$ with $s+c(m+n)\leq s'$,
\[
\E{D}\left[\log\frac{1}{\mu_C(x)}\right]\geq n-2nm\epsilon,\qquad\text{equivalently}\qquad \KL(D\|\mu_C)\geq\KL(D\|U_n)-2nm\epsilon.
\]
\end{corollary}

\begin{proof}
Let $X=G(U_\ell)$ and let $Y_1,\dots,Y_n$ be independent uniform bits, independent of $X$, so $\sum_i\H(Y_i\mid X_{<i})=n$. Fix a circuit $A$ of size at most $s'$ and an index $i$, and let $A_i(z)=A(i,z)$ on $z\in\{0,1\}^i$. The circuit $r\mapsto A_i(r_{\leq i})$ on $\{0,1\}^n$ has size at most $s'$ and distinguishes $X$ from $U_n$ with advantage $|\Pr[A_i(X_{\leq i})=1]-\Pr[A_i(U_i)=1]|$, which is therefore at most $\epsilon$. Similarly, for each fixed bit $\beta$ the circuit $r\mapsto A_i(r_{<i}\beta)$ has advantage at most $\epsilon$, and averaging over $\beta$ shows $|\Pr[A_i(X_{<i}Y_i)=1]-\Pr[A_i(U_i)=1]|\leq\epsilon$. So for each $i$ the advantage of $A_i$ between $X_{<i}X_i$ and $X_{<i}Y_i$ is at most $2\epsilon$, and the same holds for the average over $i$. Thus $D$ has $(s',2\epsilon)$-next-bit pseudoentropy at least $n$, and Theorem~\ref{nbpelower} gives the first inequality. The second follows from $\KL(D\|\mu_C)=\E{D}[\log 1/\mu_C(x)]-\H(D)$ and $\KL(D\|U_n)=n-\H(D)$.
\end{proof}

Corollary~\ref{prgcecor} complements Theorem~\ref{prgthm}: no efficient model does noticeably better than uniform on a pseudorandom distribution, and by Theorem~\ref{prgthm} any efficient model that does as well as uniform is close to uniform.

\subsection{Pseudoentropy Upper Bounds the Loss}

The converse uses the following characterization of Vadhan and Zheng. For jointly distributed $(Z,B)$ with $B\in\{0,1\}$ and a circuit $P$ computing positive values $P(z,0),P(z,1)$, let $C_P$ be the random bit with $\Pr[C_P=b\mid Z=z]=P(z,b)/(P(z,0)+P(z,1))$. Say that $B$ is \emph{$(t,\delta)$-KL-hard given $Z$} if for every circuit $P$ of size at most $t$, $\KL\big((Z,B)\,\|\,(Z,C_P)\big)>\delta$~\cite[Definition 3.2]{VZ12}. Say that $B$ has \emph{$(t,\epsilon)$-conditional pseudoentropy at least $k$ given $Z$} if there is a random variable $B'$ jointly distributed with $Z$ such that $\H(B'\mid Z)\geq k$ and no circuit of size at most $t$ distinguishes $(Z,B)$ from $(Z,B')$ with advantage greater than $\epsilon$~\cite[Definition 2.9]{VZ12}.

\begin{theorem}[{Vadhan and Zheng~\cite[Theorem 3.8]{VZ12}}]
\label{vzthm}
Let $(Z,B)$ be jointly distributed with $Z\in\{0,1\}^j$ and $B\in\{0,1\}$. If $B$ is $(t,\delta)$-KL-hard given $Z$, then for every $\epsilon>0$, $B$ has $(t',\epsilon)$-conditional pseudoentropy at least $\H(B\mid Z)+\delta-\epsilon$ given $Z$, where $t'=t^{\Omega(1)}/\poly(j,1/\epsilon)$.
\end{theorem}

Theorem~\ref{vzthm} allows predictors that output arbitrary positive values, while our predictors output probabilities with $m$ bits of precision bounded away from $0$ and $1$. The following lemma, proved in Section~\ref{convsec}, shows this restriction costs little. We assume circuits output positive rationals as pairs of binary integers, so a circuit of size $t$ outputs numbers of at most $t$ bits.

\begin{lemma}
\label{convlem}
There is a constant $c$ such that the following holds. Let $(Z,B)$ be jointly distributed with $B\in\{0,1\}$, let $P$ be a circuit of size $t$ computing positive values $P(z,0),P(z,1)$, let $\gamma=2^{-g}$ for an integer $g\geq 2$, and let $m\geq 2g+1$. There is a circuit $C$ of size at most $c(t+m)^2$ whose outputs are multiples of $2^{-m}$ in $[2^{-m},1-2^{-m}]$ such that
\[
\Ex\big[\ell(C(Z),B)\big]\leq\H(B\mid Z)+\KL\big((Z,B)\,\|\,(Z,C_P)\big)+8\gamma.
\]
\end{lemma}

\begin{theorem}
\label{nbpeupper}
There is a constant $c$ such that the following holds. Let $D$ be a distribution on $\{0,1\}^n$, let $\gamma=2^{-g}$ for an integer $g\geq 2$, let $m\geq 2g+1$ and $\epsilon>0$, and let $t$ and $s$ satisfy $s\geq c(t+m)^2$. Let $s^*=n(s+c(m+n))$. Then $D$ has $(t',\epsilon)$-next-bit pseudoentropy at least $\CE_{s^*,m}(D)-n(9\gamma+\epsilon)$, where $t'=t^{\Omega(1)}/\poly(n,1/\epsilon)$ is as in Theorem~\ref{vzthm}.
\end{theorem}

We prove Theorem~\ref{nbpeupper} in Section~\ref{upperproofsec}.

\subsection{The Characterization}

We now state the asymptotic form. Let $D=\{D_n\}$ be a family of distributions with $D_n$ on $\{0,1\}^n$ and let $k:\mathbb{N}\to\mathbb{R}$. We say $D$ has \emph{efficient cross-entropy at least $k$} if for all polynomials $s,m,p$ and all sufficiently large $n$, $\CE_{s(n),m(n)}(D_n)\geq k(n)-1/p(n)$. We say $D$ has \emph{next-bit pseudoentropy at least $k$} if for all polynomials $s,p$ and all sufficiently large $n$, $D_n$ has $(s(n),1/p(n))$-next-bit pseudoentropy at least $k(n)-1/p(n)$.

\begin{corollary}
\label{charcor}
A family $D$ has efficient cross-entropy at least $k$ if and only if it has next-bit pseudoentropy at least $k$.
\end{corollary}

The proof, which chooses parameters in Theorems~\ref{nbpelower} and~\ref{nbpeupper}, is in Section~\ref{upperproofsec}.

So the best loss achievable by any efficient model is the next-bit pseudoentropy of the data. The model absorbs everything that efficient observers can predict and pays full price, one bit per bit of apparent randomness, for everything else. Since next-bit pseudoentropy depends on the order of the bits, so does the best loss (Proposition~\ref{orderprop}).

\subsection{Proofs of Theorem~\ref{nbpeupper} and Corollary~\ref{charcor}}
\label{upperproofsec}

\begin{proof}[Proof of Theorem~\ref{nbpeupper}]
Let $X\sim D$. For each $i$, let $\lambda_i$ be the minimum of $\Ex[\ell(C(X_{<i}),X_i)]$ over circuits $C$ of size at most $s$ on $\{0,1\}^{i-1}$ with outputs multiples of $2^{-m}$ in $[2^{-m},1-2^{-m}]$, and let $\delta_i=\lambda_i-\H(X_i\mid X_{<i})$, which is nonnegative by Gibbs' inequality.

First, $\sum_i\lambda_i\geq\CE_{s^*,m}(D)$. Indeed, combining the $n$ minimizing circuits into one predictor that selects the circuit for position $|y|+1$ on input $y$ gives a predictor in $\Pred_{s^*,m}$ for a suitable constant $c$, and by Proposition~\ref{nexttokenprop} its cross-entropy is $\sum_i\lambda_i$. So
\[
\H(D)+\sum_{i=1}^n\delta_i=\sum_{i=1}^n\lambda_i\geq\CE_{s^*,m}(D).
\]

Next, fix $i$ and suppose $\delta_i>9\gamma$. We claim $X_i$ is $(t,\delta_i-9\gamma)$-KL-hard given $X_{<i}$. Otherwise some circuit $P$ of size $t$ has $\KL\big((X_{<i},X_i)\,\|\,(X_{<i},C_P)\big)\leq\delta_i-9\gamma$, and Lemma~\ref{convlem} gives a circuit of size at most $s$ with outputs of the required form and loss at most $\H(X_i\mid X_{<i})+\delta_i-\gamma<\lambda_i$, contradicting the definition of $\lambda_i$. By Theorem~\ref{vzthm} there is $Y_i$ jointly distributed with $X_{<i}$ such that $\H(Y_i\mid X_{<i})\geq\H(X_i\mid X_{<i})+\delta_i-9\gamma-\epsilon$ and $(X_{<i},X_i)$ and $(X_{<i},Y_i)$ are $(t',\epsilon)$-indistinguishable. If instead $\delta_i\leq 9\gamma$, let $Y_i=X_i$, which satisfies the same two conditions. In either case extend $Y_i$ to be jointly distributed with $X$ by making it conditionally independent of $X_{\geq i}$ given $X_{<i}$.

Summing, $\sum_i\H(Y_i\mid X_{<i})\geq\H(D)+\sum_i\delta_i-n(9\gamma+\epsilon)\geq\CE_{s^*,m}(D)-n(9\gamma+\epsilon)$. For indistinguishability, a circuit $A$ of size $t'$ taking $(i,z)$ yields, for each fixed $i$, a circuit $A(i,\cdot)$ of size at most $t'$, whose advantage is at most $\epsilon$. Averaging over $i$ gives advantage at most $\epsilon$.
\end{proof}

\begin{proof}[Proof of Corollary~\ref{charcor}]
Suppose $D$ has next-bit pseudoentropy at least $k$, and fix polynomials $s,m,p$. Let $p'=(nm+1)p$ and apply the hypothesis with circuit size $s+c(m+n)$ and $1/p'$. Theorem~\ref{nbpelower} gives $\CE_{s,m}(D_n)\geq k-1/p'-nm/p'\geq k-1/p$.

Conversely, suppose $D$ has efficient cross-entropy at least $k$, and fix polynomials $s_0$ and $p$. Let $g=\lceil\log(36np)\rceil$, $\gamma=2^{-g}\leq 1/(36np)$, $\epsilon=1/(4np)$ and $m=2g+1=O(\log n)$. Choose a polynomial $t$ large enough that $t'=t^{\Omega(1)}/\poly(n,1/\epsilon)\geq s_0$, and let $s$ and $s^*$ be as in Theorem~\ref{nbpeupper}, both polynomial. Since $\Pred_{s,m}\subseteq\Pred_{s,m'}$ whenever $m\leq m'$, the hypothesis applied with a polynomial $m'\geq m$ gives $\CE_{s^*,m}(D_n)\geq\CE_{s^*,m'}(D_n)\geq k-1/(2p)$ for large $n$, and Theorem~\ref{nbpeupper} gives $(s_0,\epsilon)$-next-bit pseudoentropy at least $k-1/(2p)-n(9\gamma+\epsilon)\geq k-1/p$, with $\epsilon\leq 1/p$.
\end{proof}

\subsection{Proof of Lemma~\ref{convlem}}
\label{convsec}

Let $p(z)=P(z,1)/(P(z,0)+P(z,1))\in(0,1)$, and for $b\in\{0,1\}$ write $p_b(z)$ for $p(z)$ if $b=1$ and $1-p(z)$ if $b=0$, and similarly for other functions. Then $\Ex[\log 1/p_B(Z)]=\H(B\mid Z)+\KL\big((Z,B)\,\|\,(Z,C_P)\big)$. Let $\tilde p(z)=(1-2\gamma)p(z)+\gamma\in[\gamma,1-\gamma]$ and $C(z)=p^*(z)=\lfloor 2^m\tilde p(z)\rfloor/2^m$. Since $\gamma=2^{-g}$, $2^m\tilde p(z)$ is a ratio of integers of $O(t+m)$ bits computable from $P(z,0)$ and $P(z,1)$, so $C$ has size at most $c(t+m)^2$ using schoolbook division. Since $2^{-m}\leq\gamma^2/2\leq\gamma/8$, and since $\gamma$ is a multiple of $2^{-m}$, we have $p^*(z)\in[\gamma,1-\gamma]\subseteq[2^{-m},1-2^{-m}]$, and $p^*_b(z)\geq\gamma/2$. For every $z$ and $b$:
\begin{enumerate}
\item $\tilde p_b(z)\geq(1-2\gamma)p_b(z)$, so $\log 1/\tilde p_b(z)\leq\log 1/p_b(z)+\log 1/(1-2\gamma)\leq\log 1/p_b(z)+6\gamma$, using $\log 1/(1-x)\leq 2x/\ln 2$ for $x\leq 1/2$.
\item $|p^*_b(z)-\tilde p_b(z)|<2^{-m}$ and both values are at least $\gamma/2$, so, using $|\log a-\log a'|\leq|a-a'|/(\min(a,a')\ln 2)$, we get $|\log 1/p^*_b(z)-\log 1/\tilde p_b(z)|\leq 2^{-m+1}/(\gamma\ln 2)\leq\gamma/\ln 2\leq 2\gamma$.
\end{enumerate}
Averaging over $(Z,B)$ gives $\Ex[\ell(C(Z),B)]=\Ex[\log 1/p^*_B(Z)]\leq\Ex[\log 1/p_B(Z)]+8\gamma$, which proves the lemma. \qed

\subsection{The Order of the Data Matters}
\label{ordersec}

Next-bit pseudoentropy depends on the order of the bits~\cite{VZ12}, and so, by Corollary~\ref{charcor}, does the best achievable loss. The following direct argument, adapted from the proof of~\cite[Lemma 4.2]{VZ12}, makes this concrete for one-way functions without going through pseudoentropy. Say $f:\{0,1\}^n\to\{0,1\}^n$ is \emph{$(t,\eta)$-one-way} if for every randomized circuit $A$ of size at most $t$, $\Pr[f(A(f(U_n)))=f(U_n)]<\eta$.

\begin{proposition}
\label{orderprop}
There is a constant $c$ such that the following holds. Let $f:\{0,1\}^n\to\{0,1\}^n$ be computed by a circuit of size $s_f$ and be $(t,\eta)$-one-way, and let $D_1$ and $D_2$ be the distributions of $(f(U_n),U_n)$ and $(U_n,f(U_n))$ on $\{0,1\}^{2n}$, both of entropy $n$.
\begin{enumerate}
\item Every $C\in\Pred_{s,m}$ with $cn(s+m+n)\leq t$ has $\Ex_{x\sim D_1}[\log 1/\mu_C(x)]>n+\log(1/\eta)$.
\item Some $C\in\Pred_{s_f+c(n^2+m),m}$ has $\Ex_{x\sim D_2}[\log 1/\mu_C(x)]\leq n+3n2^{-m}$.
\end{enumerate}
\end{proposition}

\begin{proof}
For part~1, write a sample of $D_1$ as $(Y,X)$ with $Y=f(X)$. The distribution $\mu_C$ factors as $\mu_C(y,x)=\mu_C(y)\mu_C(x\mid y)$, where $\mu_C(y)$ is its marginal on the first $n$ bits. By Gibbs' inequality, $\Ex[\log 1/\mu_C(Y)]\geq\H(Y)$. Let $S$ be the randomized circuit that, given $y$, samples $x$ bit by bit from the conditional probabilities computed by $C$, using $m$ random bits per position to sample each $m$-bit probability exactly. Then $S(y)$ has distribution $\mu_C(\cdot\mid y)$, $S$ has size at most $cn(s+m+n)\leq t$, and
\[
\Ex[\log 1/\mu_C(X\mid Y)]=\H(X\mid Y)+\KL\big((Y,X)\,\|\,(Y,S(Y))\big).
\]
Let $h(y,x)=1$ if $f(x)=y$ and $0$ otherwise. Applying $h$ cannot increase KL divergence, $h(Y,X)=1$ always, and $h(Y,S(Y))=1$ with probability $q=\Pr[f(S(f(U_n)))=f(U_n)]<\eta$. So the KL divergence is at least $\KL(\mathrm{Bern}(1)\|\mathrm{Bern}(q))=\log(1/q)>\log(1/\eta)$. Adding up, $\Ex[\log 1/\mu_C(Y,X)]>\H(Y)+\H(X\mid Y)+\log(1/\eta)=n+\log(1/\eta)$.

For part~2, let $C(y)=1/2$ for $|y|<n$, and for $|y|=n+j-1$ with $1\leq j\leq n$ let $C(y)=1-2^{-m}$ if the $j$-th bit of $f(y_1\cdots y_n)$ is $1$ and $2^{-m}$ otherwise. The first $n$ positions contribute exactly $n$ to the loss and each later position contributes $\log 1/(1-2^{-m})\leq 3\cdot 2^{-m}$.
\end{proof}

When $f$ is one-way against polynomial-size circuits, $\eta$ is negligible and $\log(1/\eta)$ is superlogarithmic. So the same information presented in two orders costs every efficient model at least $\log(1/\eta)-3n2^{-m}$ more bits in one order than the best efficient model pays in the other.

\bibliographystyle{alpha}
\bibliography{mlcomplexity}

\end{document}